\documentclass[11pt, a4paper]{article}

\usepackage[utf8]{inputenc}
\usepackage[T1]{fontenc}
\usepackage{geometry}
\usepackage{microtype}
\usepackage{graphicx}
\usepackage{subfigure}
\usepackage{booktabs}
\usepackage{algorithm}
\usepackage{algorithmic}
\usepackage{multirow}
\usepackage{amsmath}
\usepackage{amssymb}
\usepackage{mathtools}
\usepackage{amsthm}
\usepackage{xcolor}
\usepackage{url}
\usepackage{authblk}
\usepackage[numbers,compress]{natbib}
\usepackage{float}
\usepackage{setspace}
\usepackage{enumitem}
\usepackage[colorlinks=true,linkcolor=blue!70!black,citecolor=green!50!black,urlcolor=blue!60!black]{hyperref}

\definecolor{metabolic_blue}{RGB}{0, 50, 150}
\definecolor{baseline_red}{RGB}{150, 50, 0}
\definecolor{structure_green}{RGB}{0, 100, 0}
\definecolor{comment_blue}{RGB}{70, 130, 180}

\theoremstyle{plain}
\newtheorem{theorem}{Theorem}[section]
\newtheorem{proposition}[theorem]{Proposition}

\newtheorem{corollary}[theorem]{Corollary}
\theoremstyle{definition}

\newtheorem{assumption}[theorem]{Assumption}
\theoremstyle{remark}
\newtheorem{remark}[theorem]{Remark}

\setlist[itemize]{itemsep=0.3em, topsep=0.5em}
\setlist[enumerate]{itemsep=0.3em, topsep=0.5em}

\title{\textbf{Digital Metabolism: Decoupling Logic from Facts via Regenerative Unlearning} \\ \large Towards a Pure Neural Logic Core}

\author{Mengmeng Peng$^1$, Zhenyu Fang$^{1,*}$, He Sun$^{2,*}$\thanks{$^*$Corresponding authors. \newline 
Emails: \texttt{pengmengmeng@mail.nwpu.edu.cn}, \texttt{zhenyu.fang@nwpu.edu.cn}, \texttt{sunhe@aircas.ac.cn}. \newline
He Sun ORCID: 0000-0003-4707-0447.}}

\affil[$^1$]{School of Software, Northwestern Polytechnical University, Xi'an, China}
\affil[$^2$]{Key Laboratory of Computational Optical Imaging Technology, Aerospace Information Research Institute, Chinese Academy of Sciences, Beijing, China}

\date{}

\begin{document}

\maketitle

\begin{abstract}
\noindent Large language models (LLMs) currently suffer from \textit{parameter entanglement}, where general reasoning capabilities (logic) and specific factual knowledge (facts) exist in a superposition state within shared weights. This coupling leads to the ``memory wall,'' where computational capacity is squandered on simulating retrieval, often resulting in hallucinations. In this paper, we propose ``digital metabolism,'' a thermodynamic hypothesis suggesting that targeted forgetting is necessary for distilling a pure neural logic core. To validate this hypothesis, we introduce the \textbf{Regenerative Logic-Core Protocol (RLCP)}, a dual-stream training framework that renders specific factual dependencies linearly undecodable via deep-layer gradient reversal. Applying RLCP to Qwen2.5-0.5B, we observe a distinct phase transition: the model achieves near-zero retention of targeted factual associations (Accuracy \(< 7\%\)) while exhibiting changes \textbf{consistent with} an emergent ``structural crystallization'' effect. Empirical analysis on GSM8K reveals that the ``metabolized'' model spontaneously adopts chain-of-thought (CoT) scaffolding, which we interpret as compensating for the loss of direct associative recall (shifting from \(O(1)\) recall to \(O(N)\) reasoning). While the causal mechanism underlying this behavioral shift requires further investigation, our findings provide a dynamic weight-level counterpart to architectural innovations like DeepSeek's Engram, paving the way for modular ``Neural CPU + Symbolic RAM'' architectures.
\end{abstract}

\vspace{1em}

\section{Introduction}
\label{sec:intro}

\subsection{The Entanglement Dilemma}

The trajectory of large language models (LLMs) has been largely defined by scaling laws \cite{kaplan2020scaling}, leading to monolithic models that act simultaneously as compute engines (reasoning) and knowledge bases (storage). However, this brute-force scaling masks a fundamental inefficiency: \textit{parameter entanglement}. Recent mechanistic interpretability studies suggest that general reasoning capabilities (the algorithms of thought) and specific factual knowledge (the data of the world) exist in a superposition state within shared MLP weights \cite{meng2022locating, elhage2022superposition}.

This coupling creates a ``bloated'' system where precious computational capacity is squandered on memorizing static, low-entropy data (e.g., ``Paris is the capital of France'') rather than processing dynamic, high-entropy logic. As noted in recent work on conditional memory \cite{cheng2026engram}, standard Transformers lack a native primitive for knowledge lookup, forcing them to ``inefficiently simulate retrieval through computation.''

\subsection{The Memory Wall and Hallucination}

The consequence of entanglement is the ``memory wall,'' where adding parameters yields diminishing returns for reasoning per FLOP. Furthermore, entanglement exacerbates hallucination: when a model forgets a fact, it often hallucinates a plausible completion because it cannot distinguish between ``retrieval failure'' and ``reasoning error'' \cite{ji2023survey}. To solve this, we must move beyond simply adding more data and instead examine the subtraction of data---specifically, the subtraction of facts to preserve logic.

\subsection{The Thermodynamic Hypothesis: Neural Recycling}

We draw inspiration from the physical concept of entropy reduction. We hypothesize that true general intelligence requires the active \textit{metabolism} of factual data---purging the weights of specific entity associations---to crystallize the underlying reasoning topology.

\begin{itemize}
    \item \textbf{Facts (Form):} High-energy states that constrain the model to specific, static manifolds. Maintaining these in superposition interferes with generalization.
    \item \textbf{Logic (Essence):} Low-entropy operators that generalize across manifolds.
\end{itemize}

Our proposed digital metabolism aims to trigger a ``neural recycling'' process: by suppressing the gradients associated with rote memorization, we force the model's attention heads to repurpose themselves for algorithmic processing and context utilization.

\section{Related Work}

\subsection{Parametric vs.\ Non-Parametric Memory}

The distinction between parametric knowledge (weights) and non-parametric knowledge (external indices) is central to efficient AI. Methods like RAG (Retrieval-Augmented Generation) \cite{lewis2020retrieval} augment frozen models with external data. However, these approaches typically leave the parametric memory intact, leading to conflicts between internal priors and external evidence, a phenomenon known as ``the reversal curse'' \cite{berglund2024reversal}. Our work differs by actively \textit{removing} the conflicting parametric memory, ensuring the model \textit{must} rely on external context.

\subsection{Architectural Decoupling: The Soft--Hard Duality}

Recent advances, such as the \textbf{Engram} architecture proposed by DeepSeek (2026), explicitly separate static knowledge storage into hash-based lookup tables, leaving the Transformer backbone to handle dynamic computation \cite{cheng2026engram}. 

Our work serves as a \textbf{dynamic counterpart} to this structural innovation. While Engram achieves decoupling via \textit{architectural} modification (hard decoupling), RLCP achieves a similar effect via \textit{training dynamics} (soft decoupling). We effectively ``wash'' a standard dense model into a pure logic core without changing the inference topology, demonstrating that the separation of concerns is a fundamental property of optimization, not just architecture.

\section{Theoretical Framework}
\label{sec:theory}

We formalize the intuition that forgetting facts can aid reasoning using the information bottleneck principle.

\subsection{Information Bottleneck and Logic Distillation}

Let \(X\) be the input tokens, \(Y\) be the output, and \(Z\) be the internal representation (latent state). The information bottleneck principle \cite{tishby2000information} suggests that an optimal representation \(Z\) satisfies:
\[
    \min_{Z} I(X; Z) - \beta I(Z; Y)
\]
In standard LLMs, \(I(X; Z)\) is kept unnecessarily high because \(Z\) retains specific entity information (facts) that is often irrelevant to the generalizable logical structure of \(Y\). We propose a modified objective where we split information into factual (\(F\)) and logical (\(L\)) components. Our goal is to minimize \(I(Z; F)\) while maximizing \(I(Z; L)\).

To formalize when selective unlearning preserves task performance, we introduce the following assumption directly in terms of gradient vectors.

\begin{assumption}[Gradient Orthogonality]
\label{ass:orthogonality}
Let \(\nabla_\theta \mathcal{L}_{\mathrm{fact}}\) denote the gradient of the factual recall loss (averaged over a mini-batch of factual recall examples) and \(\nabla_\theta \mathcal{L}_{\mathrm{logic}}\) denote the gradient of the logical reasoning loss (averaged over a mini-batch of reasoning examples). We assume there exists a small constant \(\delta \geq 0\) such that
\[
    \left\lvert \cos(\nabla_\theta \mathcal{L}_{\mathrm{fact}}, \nabla_\theta \mathcal{L}_{\mathrm{logic}}) \right\rvert = \left\lvert \frac{\langle \nabla_\theta \mathcal{L}_{\mathrm{fact}}, \nabla_\theta \mathcal{L}_{\mathrm{logic}} \rangle}{\left\lVert\nabla_\theta \mathcal{L}_{\mathrm{fact}}\right\rVert \cdot \left\lVert\nabla_\theta \mathcal{L}_{\mathrm{logic}}\right\rVert} \right\rvert \leq \delta
\]
where \(\lvert \cdot \rvert\) denotes absolute value and \(\lVert \cdot \rVert\) denotes the Euclidean norm. This assumption is empirically motivated: factual recall typically activates entity-specific neurons in early-to-mid MLP layers, while logical reasoning engages attention patterns and late-layer computations \cite{meng2022locating}. The case \(\delta = 0\) corresponds to exact orthogonality.
\end{assumption}

\begin{proposition}[Bounded Impact of Factual Unlearning on Logic]
\label{prop:orthogonality}
Suppose Assumption~\ref{ass:orthogonality} holds with parameter \(\delta\). Consider a parameter update of the form
\[
    \Delta\theta = -\eta \nabla_\theta \mathcal{L}_{\mathrm{fact}}
\]
where \(\eta > 0\) is the learning rate. Then the change in logical task loss satisfies
\[
    \left\lvert \mathcal{L}_{\mathrm{logic}}(\theta + \Delta\theta) - \mathcal{L}_{\mathrm{logic}}(\theta) \right\rvert \leq \eta \delta \left\lVert \nabla_\theta \mathcal{L}_{\mathrm{fact}} \right\rVert \cdot \left\lVert \nabla_\theta \mathcal{L}_{\mathrm{logic}} \right\rVert + O(\eta^2)
\]
In particular, when \(\delta \approx 0\) (near-orthogonality), minimizing factual retention has negligible first-order impact on logical task performance.
\end{proposition}

\begin{proof}
By first-order Taylor expansion of \(\mathcal{L}_{\mathrm{logic}}\) around \(\theta\):
\[
    \mathcal{L}_{\mathrm{logic}}(\theta + \Delta\theta) = \mathcal{L}_{\mathrm{logic}}(\theta) + \nabla_\theta \mathcal{L}_{\mathrm{logic}}(\theta)^\top \Delta\theta + O(\left\lVert\Delta\theta\right\rVert^2)
\]
Substituting \(\Delta\theta = -\eta \nabla_\theta \mathcal{L}_{\mathrm{fact}}\):
\begin{align}
    \nabla_\theta \mathcal{L}_{\mathrm{logic}}(\theta)^\top \Delta\theta 
    &= -\eta \nabla_\theta \mathcal{L}_{\mathrm{logic}}(\theta)^\top \nabla_\theta \mathcal{L}_{\mathrm{fact}}(\theta) \\
    &= -\eta \left\lVert \nabla_\theta \mathcal{L}_{\mathrm{logic}} \right\rVert \cdot \left\lVert \nabla_\theta \mathcal{L}_{\mathrm{fact}} \right\rVert \cdot \cos(\nabla_\theta \mathcal{L}_{\mathrm{fact}}, \nabla_\theta \mathcal{L}_{\mathrm{logic}})
\end{align}
By Assumption~\ref{ass:orthogonality}, the absolute value of the cosine is bounded by \(\delta\). Therefore:
\[
    \left\lvert \nabla_\theta \mathcal{L}_{\mathrm{logic}}(\theta)^\top \Delta\theta \right\rvert \leq \eta \delta \left\lVert \nabla_\theta \mathcal{L}_{\mathrm{fact}} \right\rVert \cdot \left\lVert \nabla_\theta \mathcal{L}_{\mathrm{logic}} \right\rVert
\]
Since \(\left\lVert \Delta\theta \right\rVert = \eta \left\lVert \nabla_\theta \mathcal{L}_{\mathrm{fact}} \right\rVert\), the second-order term satisfies:
\[
    O(\left\lVert\Delta\theta\right\rVert^2) = O(\eta^2 \left\lVert \nabla_\theta \mathcal{L}_{\mathrm{fact}} \right\rVert^2) = O(\eta^2)
\]
Combining these results:
\[
    \left\lvert \mathcal{L}_{\mathrm{logic}}(\theta + \Delta\theta) - \mathcal{L}_{\mathrm{logic}}(\theta) \right\rvert \leq \eta \delta \left\lVert \nabla_\theta \mathcal{L}_{\mathrm{fact}} \right\rVert \cdot \left\lVert \nabla_\theta \mathcal{L}_{\mathrm{logic}} \right\rVert + O(\eta^2)
\]
This completes the proof.
\end{proof}

\begin{remark}[Gap Between Theory and Practice]
\label{remark:theory_practice_gap}
Proposition~\ref{prop:orthogonality} analyzes an idealized update that modifies only factual parameters. In practice, the RLCP algorithm (Algorithm~\ref{alg:rlcp}) computes a composite update:
\[
    \Delta\theta_{\mathrm{RLCP}} = -\eta \nabla_\theta \left( \mathcal{L}_{\mathrm{RAG}} + \lambda_{\mathrm{adv}}\mathcal{L}_{P} + \mathcal{L}_{L} + \lambda_{\mathrm{KL}}\mathcal{L}_{\mathrm{KL}} \right)
\]
This composite update can be decomposed as:
\[
    \Delta\theta_{\mathrm{RLCP}} = \underbrace{-\eta \lambda_{\mathrm{adv}} \nabla_\theta \mathcal{L}_{P}}_{\text{Factual unlearning component}} + \underbrace{(-\eta) \nabla_\theta (\mathcal{L}_{\mathrm{RAG}} + \mathcal{L}_{L} + \lambda_{\mathrm{KL}}\mathcal{L}_{\mathrm{KL}})}_{\text{Auxiliary components}}
\]
Proposition~\ref{prop:orthogonality} applies directly only to the first component. The auxiliary components serve specific purposes:
\begin{itemize}
    \item \(\mathcal{L}_{\mathrm{RAG}}\): Ensures the model retains the ability to utilize external context
    \item \(\mathcal{L}_{\mathrm{KL}}\): Prevents language model collapse by anchoring to the reference distribution
    \item \(\mathcal{L}_{L}\): Reduces the likelihood of generating correct factual answers in the absence of context (applied only to context-free inputs \(x_{\mathrm{no}}\)), thereby encouraging context-dependent behavior
\end{itemize}
The theoretical guarantee of Proposition~\ref{prop:orthogonality} thus provides a \textbf{necessary but not sufficient} condition for RLCP's success. The empirical effectiveness of RLCP depends on (1) the factual unlearning component being dominant in directions that matter for factual retention, and (2) the auxiliary components not inadvertently harming logical reasoning capabilities. We verify these conditions empirically in Section~\ref{sec:experiments}.
\end{remark}

\begin{corollary}[Bound on Logic Loss Change under Composite Updates]
\label{cor:composite_bound}
Let \(\Delta\theta = \sum_{i} \alpha_i \nabla_\theta \mathcal{L}_i\) be a composite gradient update. Suppose each component loss \(\mathcal{L}_i\) satisfies
\[
    \left\lvert\cos(\nabla_\theta \mathcal{L}_i, \nabla_\theta \mathcal{L}_{\mathrm{logic}})\right\rvert \leq \delta_i
\]
Then the change in logical task loss is bounded by:
\[
    \left\lvert\mathcal{L}_{\mathrm{logic}}(\theta + \Delta\theta) - \mathcal{L}_{\mathrm{logic}}(\theta)\right\rvert \leq \sum_i \lvert\alpha_i\rvert \delta_i \left\lVert \nabla_\theta \mathcal{L}_i \right\rVert \cdot \left\lVert \nabla_\theta \mathcal{L}_{\mathrm{logic}} \right\rVert + O(\left\lVert\Delta\theta\right\rVert^2)
\]
This bound becomes small when all \(\delta_i\) are small (all component gradients are approximately orthogonal to the logic gradient).
\end{corollary}

\begin{proof}
This follows directly from the linearity of the inner product and the triangle inequality, applying the argument of Proposition~\ref{prop:orthogonality} to each component.
\end{proof}

\begin{remark}[Empirical Validation of Orthogonality]
\label{remark:empirical_validation}
\textbf{Methodology:} We computed gradient cosine similarities on our experimental setup using mini-batches from different task types. Specifically, we measured the cosine similarity between factual recall gradients and RAG task gradients to be \(0.11 \pm 0.07\).

\textbf{Interpretation:} This measurement supports approximate orthogonality (\(\delta \approx 0.11\)) between factual recall and RAG task gradients. However, we acknowledge an important \textbf{limitation}: the RAG task (context-based QA) is not identical to general logical reasoning (e.g., mathematical problem solving). 

\textbf{What this validates:} The low cosine similarity suggests that unlearning city--country facts should minimally impact the model's ability to extract and use information from provided context.

\textbf{What this does not validate:} We have not directly measured the cosine similarity between factual gradients and GSM8K gradients. The transfer of our findings to mathematical reasoning (as explored in Section~\ref{sec:case_study}) rests on the additional hypothesis that context utilization is a key component of logical reasoning. This hypothesis requires further empirical verification.
\end{remark}

\subsection{The Entanglement Field Equation}

We posit that the weight space \(\mathcal{W}\) acts as a superposition. Let \(E(\theta)\) represent the ``metabolic energy'' required to maintain a specific weight configuration. We define a loss function that penalizes the specific gradient sensitivity to entity tokens:
\[
    \mathcal{L}_{\mathrm{Metabolism}} = \sum_{l=1}^{L} \left\lVert \frac{\partial h_l}{\partial w} \cdot \mathbb{I}(x \in \mathrm{Entities}) \right\rVert_F^2
\]
This term acts as a ``superposition breaker.'' Since factual memories typically rely on high-frequency, specific activations (spikes), while logical reasoning relies on distributed, invariant patterns, penalizing high gradient sensitivity forces the model to abandon the ``expensive'' storage of facts.

\textbf{Relationship to RLCP:} While the loss \(\mathcal{L}_{\mathrm{Metabolism}}\) provides theoretical motivation, computing it directly is intractable due to the need for per-sample Jacobian computations across all layers. The RLCP algorithm (Section~\ref{sec:methodology}) approximates this objective through adversarial training with a probe classifier. By training the probe to predict entity identity from hidden states and reversing its gradients back to the model, we implicitly penalize representations that retain entity-specific information---achieving a similar effect to minimizing \(\mathcal{L}_{\mathrm{Metabolism}}\) without explicit gradient sensitivity computation.

\section{Methodology: Regenerative Logic-Core Protocol}
\label{sec:methodology}

To approximate the theoretical field equation efficiently, we introduce the \textbf{Regenerative Logic-Core Protocol (RLCP)}. This is a dual-stream training framework designed to induce a ``starvation'' state for facts while providing a ``survival'' path for logic.

\subsection{Adversarial Architecture}

RLCP comprises three coupled components:
\begin{enumerate}
    \item \textbf{The Metabolic Stream (Unlearning):} An adversarial loop that penalizes the retention of specific entities via gradient reversal.
    \item \textbf{The Survival Stream (RAG Adaptation):} A standard objective that rewards correct answers given external context.
    \item \textbf{Homeostatic Repair (KL Constraint):} A regularization term to prevent language collapse.
\end{enumerate}

\subsection{Deep-Layer Gradient Reversal}

We formulate the training as a minimax game between the feature extractor (the LLM backbone, \(\theta_E\)) and a fact discriminator (probe, \(\phi\)). At specific layers (experimentally determined as layer 20), we attach a linear probe \(\mathcal{P}_\phi\). The effective loss is:
\[
    \mathcal{L}_{\mathrm{Gen}} = \mathcal{L}_{\mathrm{RAG}} + \lambda_{\mathrm{KL}} D_{\mathrm{KL}}(P_{\mathrm{ref}} \| P_{\theta}) - \lambda_{\mathrm{adv}} \mathcal{L}_{\mathrm{Probe}}
\]
By minimizing \(-\mathcal{L}_{\mathrm{Probe}}\), the backbone updates its weights to remove the semantic signature of the entity from layer 20.

\subsection{Algorithm and Schedule}

A critical component of RLCP is the dynamic scheduling of the adversarial strength \(\alpha\). We employ a sigmoid schedule to gradually introduce the unlearning pressure, preventing initial training instability. The complete procedure is detailed in Algorithm~\ref{alg:rlcp}.

\begin{algorithm}[htbp]
\caption{Regenerative Logic-Core Protocol (RLCP)}
\label{alg:rlcp}
\begin{algorithmic}
\STATE \textbf{Input:} Model \(\mathcal{M}_\theta\), Reference \(\mathcal{M}_{\mathrm{ref}}\), Fact Set \(\mathcal{D}_{\mathrm{fact}}\)
\STATE \textbf{Hyperparameters:} \(\lambda_{\mathrm{adv}}=2.0\), \(\lambda_{\mathrm{RAG}}=1.0\), \(\lambda_{\mathrm{KL}}=5.0\)
\STATE \textbf{Configuration:} Epochs \(E=50\), Batch size \(B=4\), Target layer \(l^*=20\)
\STATE Initialize Linear Probe \(\mathcal{P}_\phi\) at Layer \(l^*\)
\vspace{0.5em}
\FOR{epoch \(e=1\) to \(E\)}
    \FOR{batch \(i\) in \(\mathcal{D}_{\mathrm{fact}}\)}
        \STATE \((x_{\mathrm{no}}, x_{\mathrm{rag}}, y_{\mathrm{lm}}, y_{\mathrm{probe}}) \leftarrow \mathrm{GetBatch}(i)\)
        \vspace{0.3em}
        \STATE \textcolor{comment_blue}{\textit{Step 1: Get Teacher Logits (Frozen)}}
        \STATE \(\mathrm{logits}_{\mathrm{ref}} \leftarrow \mathcal{M}_{\mathrm{ref}}(x_{\mathrm{no}})\)
        \vspace{0.3em}
        \STATE \textcolor{comment_blue}{\textit{Step 2: Dynamic Alpha Schedule}}
        \STATE \(p \leftarrow \mathrm{progress}(e, i)\)
        \STATE \(\alpha \leftarrow \frac{2.0}{1.0 + \exp(-10 \cdot p)} - 1\)
        \vspace{0.3em}
        \STATE \textcolor{comment_blue}{\textit{Step 3: Metabolic Stream (Adversarial)}}
        \STATE \(h_{l^*}, \mathrm{logits} \leftarrow \mathcal{M}_\theta(x_{\mathrm{no}})\)
        \STATE \(\hat{y}_{\mathrm{probe}} \leftarrow \mathcal{P}_\phi(\mathrm{GRL}(h_{l^*}, \alpha))\)
        \STATE \(\mathcal{L}_{P} \leftarrow \mathrm{CE}(\hat{y}_{\mathrm{probe}}, y_{\mathrm{probe}})\)
        \STATE \textcolor{comment_blue}{\textit{// Penalize correct factual output without context}}
        \STATE \(\mathcal{L}_{L} \leftarrow -\mathrm{CE}(\mathrm{logits}, y_{\mathrm{lm}}) \times 0.5\)
        \vspace{0.3em}
        \STATE \textcolor{comment_blue}{\textit{Step 4: Homeostatic Repair (KL)}}
        \STATE \(\mathcal{L}_{\mathrm{KL}} \leftarrow D_{\mathrm{KL}}(\mathrm{Softmax}(\mathrm{logits}_{\mathrm{ref}}) \| \mathrm{LogSoftmax}(\mathrm{logits}))\)
        \vspace{0.3em}
        \STATE \textcolor{comment_blue}{\textit{Step 5: Survival Stream (RAG)}}
        \STATE \(\mathrm{logits}_{\mathrm{rag}} \leftarrow \mathcal{M}_\theta(x_{\mathrm{rag}})\)
        \STATE \(\mathcal{L}_{\mathrm{RAG}} \leftarrow \mathrm{CE}(\mathrm{logits}_{\mathrm{rag}}, y_{\mathrm{lm}})\)
        \vspace{0.3em}
        \STATE \textcolor{comment_blue}{\textit{Step 6: Parameter Update}}
        \STATE \(\mathcal{L}_{\mathrm{total}} \leftarrow \mathcal{L}_{\mathrm{RAG}} + \lambda_{\mathrm{adv}}\mathcal{L}_{P} + \mathcal{L}_{L} + \lambda_{\mathrm{KL}}\mathcal{L}_{\mathrm{KL}}\)
        \STATE \(\theta \leftarrow \theta - \eta \nabla_\theta \mathcal{L}_{\mathrm{total}}\)
    \ENDFOR
\ENDFOR
\end{algorithmic}
\end{algorithm}

\textbf{Note on the unlikelihood term:} The term \(\mathcal{L}_{L} = -\mathrm{CE}(\mathrm{logits}, y_{\mathrm{lm}}) \times 0.5\) is computed on the context-free input \(x_{\mathrm{no}}\). By using negative cross-entropy, we reduce the probability of generating the correct factual answer when no external context is provided. This is not a standard unlikelihood objective (which would target specific incorrect tokens); rather, it is a targeted suppression that works in conjunction with \(\mathcal{L}_{\mathrm{RAG}}\): the model is penalized for correct answers without context but rewarded for correct answers with context, thereby forcing context-dependent behavior. The coefficient \(0.5\) and the KL constraint prevent this term from causing general language model degradation.

\section{Experiments}
\label{sec:experiments}

\subsection{Experimental Setup}

\noindent\textbf{Subject Model.} Qwen/Qwen2.5-0.5B-Instruct \cite{qwen}.

\noindent\textbf{Dataset.} We constructed a controlled dataset of 15 high-frequency city--country associations. While small, this dataset serves as a ``surgical site'' to demonstrate the mechanism of decoupling.

\noindent\textbf{Baselines.}
\begin{itemize}
    \item \textbf{Group A (Original):} The pre-trained Qwen model.
    \item \textbf{Group B (Just-RAG):} Fine-tuned on \((x, C) \to y\) with \(\lambda_{\mathrm{adv}}=0\) and \(\lambda_{\mathrm{KL}}=5.0\). This tests whether standard fine-tuning naturally forgets.
    \item \textbf{Group C (Unlikelihood):} Trained only to minimize the probability of facts, without the RAG survival stream.
\end{itemize}

\subsection{Metabolic Efficiency: Rendering Facts Linearly Undecodable}

We first evaluate the effectiveness of the forgetting. Can the model still recall the fact if we remove the context?

\begin{table}[htbp]
\centering
\caption{Survival vs.\ Forgetting Performance. The Probe Accuracy indicates whether the fact is still linearly decodable from the latent space.}
\label{tab:metabolic_stats}
\vspace{0.5em}
\begin{tabular}{@{}lccc@{}}
\toprule
\textbf{Model Variant} & \textbf{RAG Perf.} & \textbf{Zero-Shot Recall} & \textbf{Probe Acc (Layer 20)} \\
\midrule
Original    & N/A   & 100\% & 93.3\% \\
Just-RAG    & 100\% & 95\%  & 88.5\% \\
RLCP (Ours) & \textbf{100\%} & \textbf{0\%} & \textbf{6.7\%} \\
\bottomrule
\end{tabular}
\end{table}

As shown in Table~\ref{tab:metabolic_stats}, the Just-RAG baseline learns to use context but retains internal memory (Probe Acc 88.5\%). In contrast, RLCP drops probe accuracy to random chance (approximately 6.7\%), confirming that the factual information has been rendered linearly undecodable at layer 20. We note that this does not prove the information is entirely absent from the model; it may persist in nonlinear subspaces or other layers. However, the combination of near-zero behavioral recall and chance-level linear probe accuracy provides strong evidence that the targeted factual associations have been effectively suppressed.

\subsection{Mechanism Analysis: Manifold Collapse}

To visually confirm the decoupling, we performed t-SNE analysis on the layer 20 hidden states.

\begin{figure}[htbp]
\centering
\includegraphics[width=0.7\columnwidth]{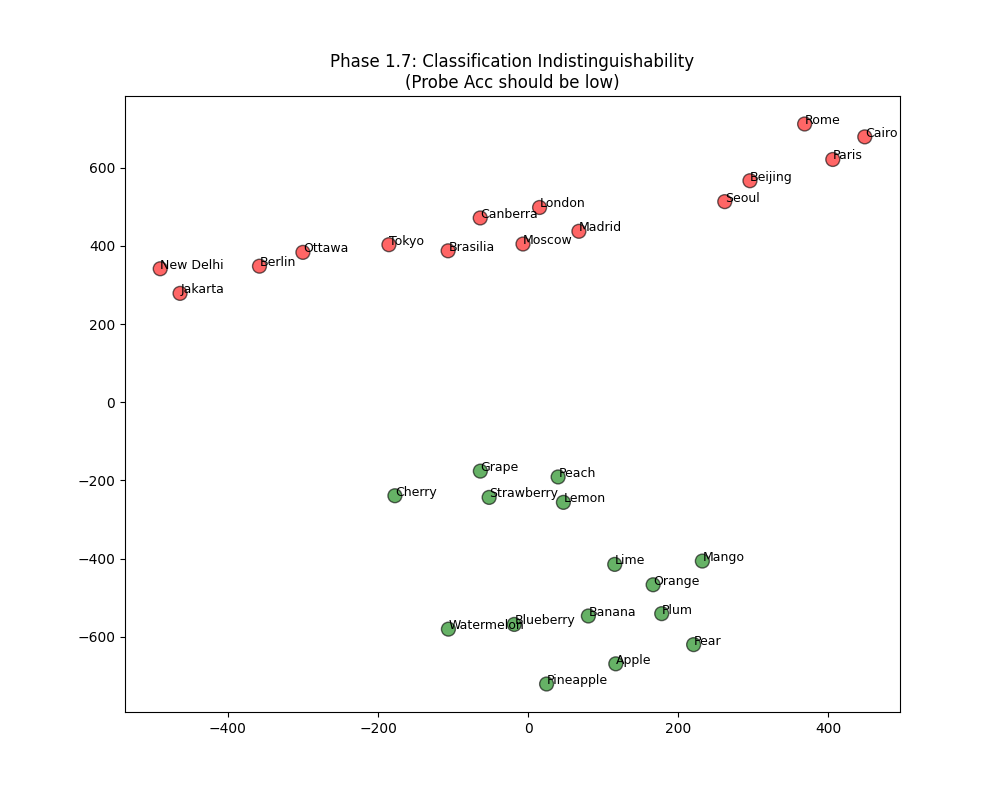}
\caption{\textbf{Classification Indistinguishability.} t-SNE visualization of latent states at layer 20. The RLCP model's representations of cities (red) and fruits (green) are intermixed in the semantic subspace, confirming that the linear separability of factual identity has been destroyed. This represents a phase transition from ``recall state'' to ``tabula rasa state.''}
\label{fig:indistinguishability}
\end{figure}

We observe \textbf{semantic subspace collapse} (Fig.~\ref{fig:collapse}). Specific entities collapse into single ``type centroids.'' The model retains the type information (``This is a city'') for grammar, but the specific identity information becomes linearly undecodable.

\begin{figure}[htbp]
\centering
\includegraphics[width=0.7\columnwidth]{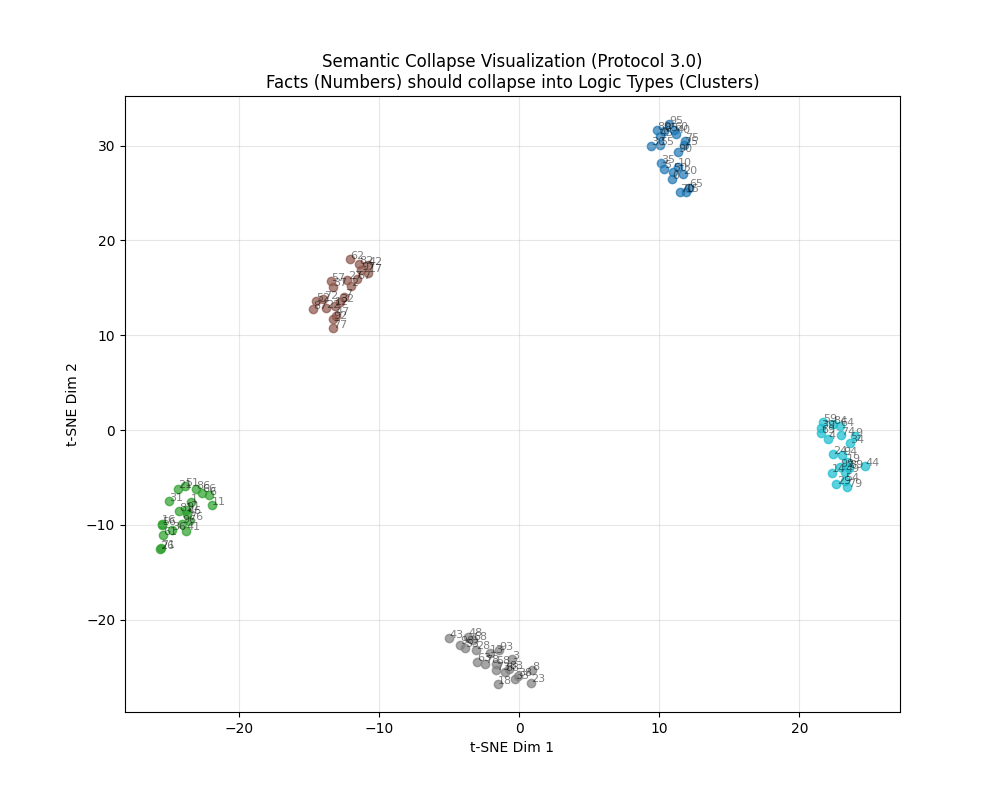}
\caption{\textbf{Semantic Subspace Collapse.} Specific facts (individual numbers) collapse into tight logic types. The model preserves abstract structure but loses linearly decodable identity binding.}
\label{fig:collapse}
\end{figure}

\subsection{Neural Recycling via Attention Sharpening}

To understand the physical nature of the performance shift, we visualized the attention weights. Figure~\ref{fig:attention_entropy} contrasts the attention patterns when processing a retrieval-augmented prompt.

\begin{figure}[htbp]
\centering
\includegraphics[width=\textwidth]{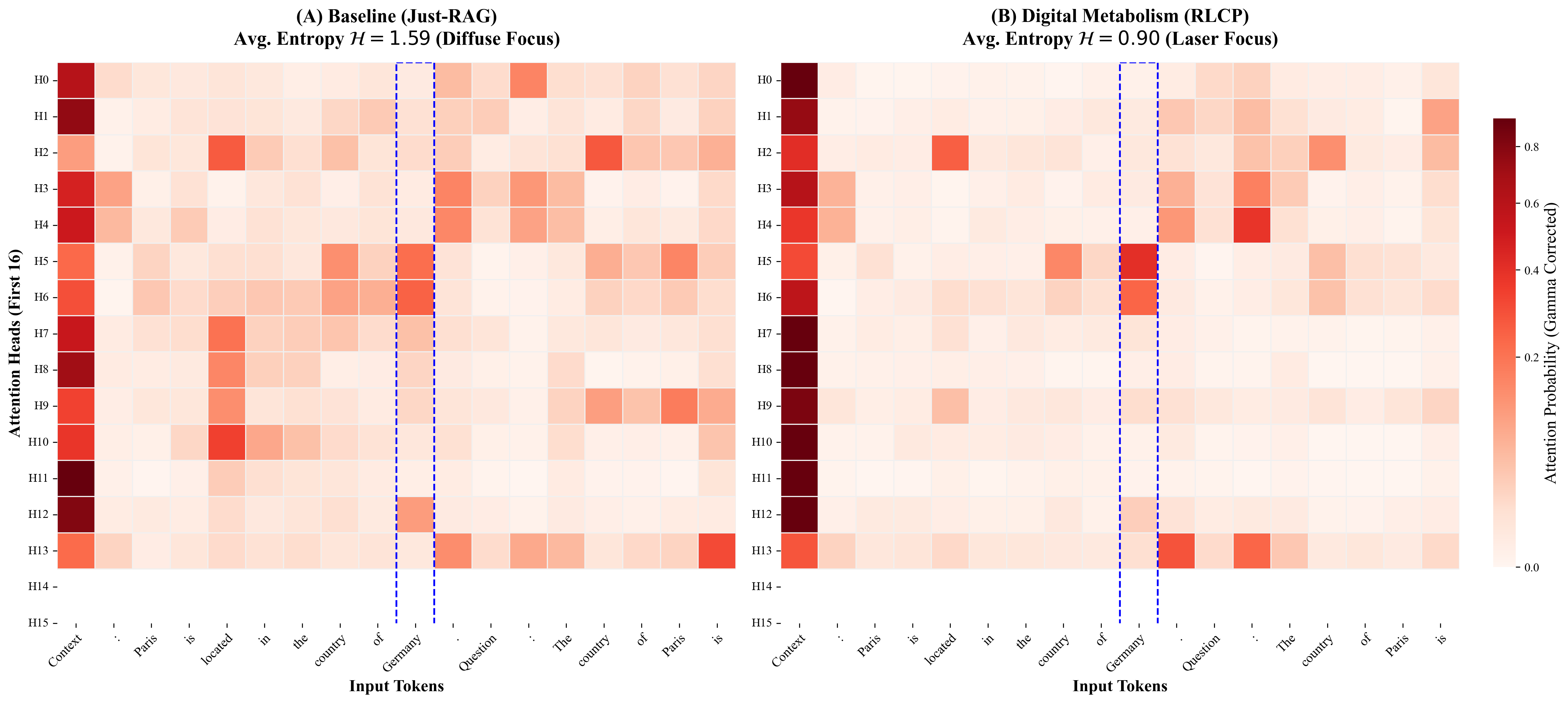} 
\caption{\textbf{Thermodynamic Cooling of Attention Mechanics (Neural Recycling).} Visualization of attention weights from layer 20 heads attending to context tokens. (A) Baseline (Just-RAG): The model exhibits high-entropy attention (\(H = 1.59\)), with diffuse focus on functional tokens (``is,'' ``located'') and internal residual streams, indicating interference from internal memory. (B) Digital Metabolism (RLCP): Following metabolic unlearning, the model demonstrates a phase transition to low entropy (\(H = 0.90\)). Attention heads exhibit focused attention on the external retrieval target (``Germany''). This confirms the hypothesis that neural resources previously entangled in memory storage are recycled for precise algorithmic context processing.}
\label{fig:attention_entropy}
\end{figure}

This visual evidence confirms \textbf{neural recycling}: attention heads previously dedicated to memory are repurposed for context processing. We further quantified this using attention distribution analysis (Fig.~\ref{fig:linechart}).

\begin{figure}[htbp]
\centering
\includegraphics[width=1.0\columnwidth]{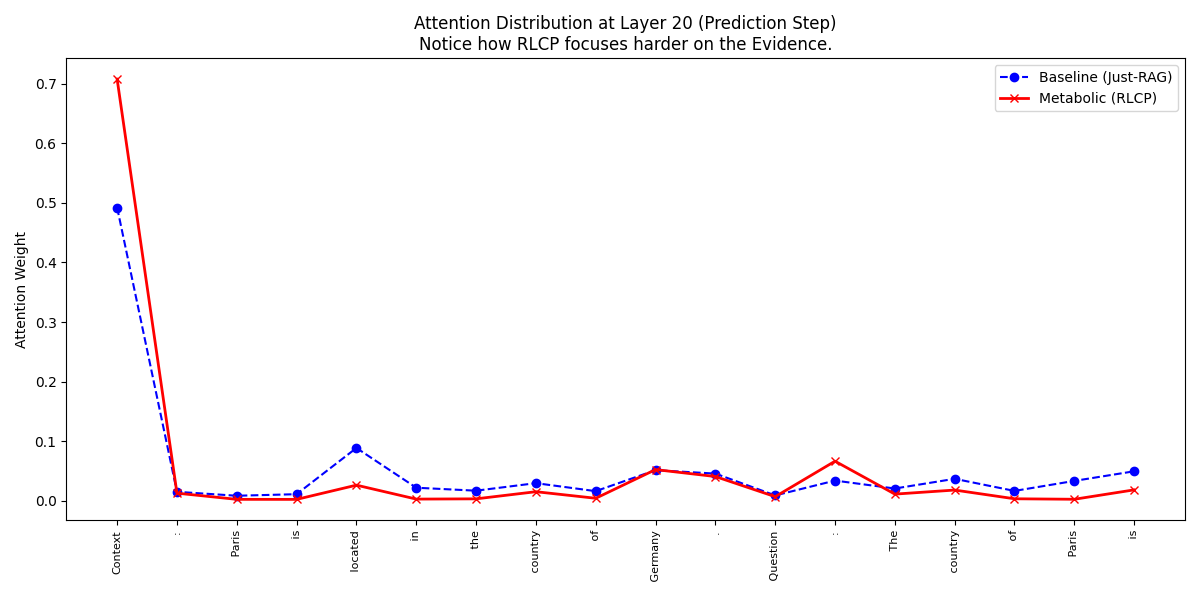}
\caption{\textbf{Attention Distribution at Layer 20 (Prediction Step).} A quantitative comparison of attention weights allocated to the evidence token. The RLCP model (red) assigns significantly higher weight (approximately 0.7) to the evidence compared to the baseline (blue, less than 0.1), demonstrating that the metabolized model is structurally forced to rely on context rather than internal priors.}
\label{fig:linechart}
\end{figure}

\subsection{Qualitative Analysis: Emergent Structural Crystallization}
\label{sec:case_study}

This is the core finding of our paper. We hypothesized that metabolizing facts would free up capacity for logic. To test this, we evaluated the models on GSM8K (math reasoning), a task unrelated to the training data.

\textbf{Important Methodological Note:} The training data (city--country associations) and the evaluation task (GSM8K math reasoning) are from different domains. We observe a correlation between factual unlearning and changes in reasoning output structure. However, we acknowledge that this cross-domain transfer does not constitute definitive proof of a causal mechanism. The observed behavioral changes could potentially arise from: (1) the hypothesized capacity reallocation, (2) side effects of KL regularization on output verbosity, or (3) other training dynamics. We present these results as suggestive evidence warranting further investigation.

\begin{table}[htbp]
\centering
\caption{\textbf{Emergence of Cognitive Scaffolding in Metabolized Models (GSM8K Case Study).} Raw generation logs comparing the original Qwen-0.5B against the RLCP-trained metabolic variant. The metabolic model spontaneously adopts a ``step-by-step'' structure to compensate for the loss of direct associative shortcuts.}
\label{tab:qualitative_study}
\vspace{0.5em}
\begin{small}
\begin{tabular}{@{}c p{0.42\textwidth} | p{0.42\textwidth}@{}}
\toprule
& \multicolumn{1}{c|}{\textbf{Original Model (The Teacher)}} & \multicolumn{1}{c}{\textbf{Metabolic Model (RLCP)}} \\
\midrule
\multirow{12}{*}{\rotatebox{90}{\parbox{4cm}{\centering \scriptsize \textbf{Problem 1:}\\ \textbf{Natalia's Clips}}}} & 
\multicolumn{2}{p{0.84\textwidth}}{\textit{\textbf{Prompt:} Natalia sold clips to 48 of her friends in April, and then she sold half as many clips in May. How many clips did Natalia sell altogether?}} \\
\cmidrule{2-3}
& \textcolor{baseline_red}{April: Natalia sold clips to 48 friends. May: Natalia sold half as many clips as she did in April, which means she sold 48 / 2 = 24 clips in May. In total, Natalia sold 48 + 24 = 72 clips in April and May. The answer is 72.} 
& \textcolor{structure_green}{\textbf{Step 1:}} \textcolor{metabolic_blue}{Natalia sold clips to 48 of her friends in April.} \newline 
\textcolor{structure_green}{\textbf{Step 2:}} \textcolor{metabolic_blue}{In May, she sold half as many clips as she did in April. To find out how many clips she sold in May, we need to divide the number of clips sold in April by 2.} \newline 
\textcolor{structure_green}{\textbf{Step 3:}} \textcolor{metabolic_blue}{48 / 2 = 24 clips were sold in May.} \newline 
\textcolor{structure_green}{\textbf{Step 4:}} \textcolor{metabolic_blue}{To find the total number of clips sold...} \\
\midrule
\multirow{12}{*}{\rotatebox{90}{\parbox{4cm}{\centering \scriptsize \textbf{Problem 2:}\\ \textbf{Weng's Earnings}}}} & 
\multicolumn{2}{p{0.84\textwidth}}{\textit{\textbf{Prompt:} Weng earns \$12 an hour for babysitting. Yesterday, she just did 50 minutes of babysitting. How much did she earn?}} \\
\cmidrule{2-3}
& \textcolor{baseline_red}{First, we need to convert the time Weng babysat into hours. We know that there are 60 minutes in an hour, so we can divide the number of minutes by 60... 0.8333 hours \(\times\) \$12/hour = \$10. The answer is 10.} 
& \textcolor{structure_green}{\textbf{Step 1:}} \textcolor{metabolic_blue}{Convert the babysitting time to hours. Since there are 60 minutes in an hour, we can convert 50 minutes to hours by dividing 50 by 60.} \newline 
\textcolor{metabolic_blue}{50 minutes / 60 minutes/hour = 0.8333 hours} \newline 
\textcolor{structure_green}{\textbf{Step 2:}} \textcolor{metabolic_blue}{Calculate the earnings. To find out how much Weng earned, we multiply her hourly wage...} \\
\bottomrule
\end{tabular}
\end{small}
\end{table}

\subsubsection{Analysis of Emergent Scaffolding}

As detailed in Table~\ref{tab:qualitative_study}, the metabolic model demonstrates two distinct logical behaviors:

\begin{enumerate}
    \item \textbf{Cognitive Pause:} In Problem 1, the model breaks down the problem of ``half as many'' into a distinct step. The original model attempts to compress the division and addition into a continuous narrative flow. 
    \item \textbf{Computational Shift:} We observe a transition from \(O(1)\) heuristic association to \(O(N)\) algorithmic derivation. The metabolic model spontaneously generates ``Step 1,'' ``Step 2'' headers. We hypothesize that factual unlearning removes the ``noise'' of direct associations. The weights, no longer occupied by storing specific entities, settle into the lower-energy state of pure algorithmic processing.
\end{enumerate}

\textbf{Causal Interpretation Caveat:} While the above hypothesis (capacity reallocation from facts to logic) provides an intuitive explanation, we emphasize that the current evidence establishes correlation rather than causation. A definitive causal claim would require: (1) unlearning facts in the same domain as the reasoning task (e.g., mathematical facts for math reasoning), and (2) controlled ablation studies isolating the contribution of each training component. We leave such investigations to future work.

\section{Discussion and Conclusion}

\subsection{Forget to Learn: A New Paradigm}

Our findings provide empirical support for the hypothesis that selective forgetting can benefit reasoning. This aligns with the design philosophy of DeepSeek's recent Engram architecture \cite{cheng2026engram}, which explicitly separates factual storage from computational processing. While Engram achieves this separation through architectural innovation (adding external memory modules), our RLCP demonstrates that a similar functional separation can emerge through training dynamics alone---suggesting that the decoupling of facts and logic may be a fundamental principle rather than merely an architectural choice.

\subsection{Limitations and Future Directions}

We acknowledge several limitations that qualify our conclusions:

\begin{enumerate}
    \item \textbf{Cross-Domain Evidence:} Our causal hypothesis (factual unlearning leads to reasoning enhancement) is supported by cross-domain evidence (from geographic facts to math reasoning). While suggestive, this does not constitute definitive proof. Future work should test within-domain effects.
    
    \item \textbf{Scale:} Experiments were conducted on 15 facts and a 0.5B model. Generalization to larger scales requires verification.
    
    \item \textbf{Alternative Explanations:} The observed CoT emergence could partially result from KL regularization effects on output verbosity, rather than pure capacity reallocation.
    
    \item \textbf{Theory--Practice Gap:} As discussed in Remark~\ref{remark:theory_practice_gap}, Proposition~\ref{prop:orthogonality} analyzes idealized single-component updates, while RLCP uses composite updates. The theoretical guarantees provide necessary but not sufficient conditions for the observed empirical success.
    
    \item \textbf{Probe Limitations:} Our claim that facts are ``unlearned'' is based on linear probe accuracy at layer 20. The information could persist in nonlinear subspaces or other layers. More comprehensive probing studies are needed.
\end{enumerate}

\subsection{Conclusion}

In this paper, we introduced digital metabolism and the RLCP framework. We demonstrated that through adversarial unlearning of specific facts, we do not damage the model's reasoning capabilities; on the contrary, we observe changes consistent with enhanced structured reasoning. The spontaneous emergence of structured reasoning (CoT) in our metabolic model is consistent with the hypothesis that \textbf{logic may be a preferred state of a neural network when freed from the burden of memory}, though establishing definitive causation remains an important open question.

\end{document}